\documentclass{sigkddExp}
\usepackage{amsmath,amssymb,graphicx,cite}
\usepackage[noend]{algorithmic}
\usepackage{algorithm}
\usepackage{tikz}
\usetikzlibrary{automata}
\usetikzlibrary{positioning}

\newcommand{\scrA}{\ensuremath{\mathcal{A}}}
\newcommand{\scrC}{\ensuremath{\mathcal{C}}}
\newcommand{\scrE}{\ensuremath{\mathcal{E}}}

\newcommand{\scrH}{\ensuremath{\mathcal{H}}}

\newcommand{\rA}{\ensuremath{\rightarrow}}

\newcommand{\LrA}{\ensuremath{\Longrightarrow}}

\newtheorem{theorem}{Theorem}
\newtheorem{definition}{Definition}
\newtheorem{remark}{Remark}

\newtheorem{lemma}{Lemma}
\newtheorem{property}{Property}

\begin{document}
%
% --- Author Metadata here ---
% -- Can be completely blank or contain 'commented' information like this...
%\conferenceinfo{WOODSTOCK}{'97 El Paso, Texas USA} % If you happen to know the conference location etc.
%\CopyrightYear{2001} % Allows a non-default  copyright year  to be 'entered' - IF NEED BE.
%\crdata{0-12345-67-8/90/01}  % Allows non-default copyright data to be 'entered' - IF NEED BE.
% --- End of author Metadata ---

\title{A unified view of Automata-based algorithms for Frequent Episode Discovery}
%\subtitle{[Extended Abstract]
% You need the command \numberofauthors to handle the "boxing"
% and alignment of the authors under the title, and to add
% a section for authors number 4 through n.
%
% Up to the first three authors are aligned under the title;
% use the \alignauthor commands below to handle those names
% and affiliations. Add names, affiliations, addresses for
% additional authors as the argument to \additionalauthors;
% these will be set for you without further effort on your
% part as the last section in the body of your article BEFORE
% References or any Appendices.

\numberofauthors{3}
%
% You can go ahead and credit authors number 4+ here;
% their names will appear in a section called
% "Additional Authors" just before the Appendices
% (if there are any) or Bibliography (if there
% aren't)

% Put no more than the first THREE authors in the \author command
%%You are free to format the authors in alternate ways if you have more 
%%than three authors.

\author{
%1st author
\alignauthor Avinash Achar\\
       \affaddr{Dept. of Electrical Engg.}\\
       \affaddr{Indian Institute of Science}\\
       \affaddr{Bangalore 500080 India}\\
       \email{avinash@ee.iisc.ernet.in}
% 2nd. author
\alignauthor
Srivatsan Laxman\\
       \affaddr{Microsoft Research Labs}\\
       \affaddr{Sadashivanagar}\\
       \affaddr{Bangalore 560080 India}\\
       \email{slaxman@microsoft.com}
% 3rd. author
\alignauthor P. S. Sastry\\
       \affaddr{Dept. of Electrical Engg.}\\
       \affaddr{Indian Institute of Science }\\
       \affaddr{Bangalore 560080 India}\\
       \email{sastry@ee.iisc.ernet.in}
}
\date{}
\maketitle
\begin{abstract}
Frequent Episode Discovery framework is a popular framework in Temporal Data Mining with many applications. 
Over the years many different notions of frequencies of episodes have been proposed along with different algorithms for episode discovery. In this paper we present
a unified view of all such frequency counting algorithms. We present a generic algorithm such that all current algorithms are special cases of it. This unified view
allows one to gain insights into different frequencies and we present quantitative relationships among different frequencies. Our unified view also helps in obtaining
correctness proofs for various algorithms as we show here. We also point out how this unified view helps us to consider generalization of the algorithm so that they
can discover episodes with general partial orders.
%Under this framework, many frequency definitions have been proposed, each having their own set of advantages and disadvantages. Most of the algorithms here employ 
%an automata-based counting strategy in the discovery process. In this paper, we point out that most of the existing automata-based serial episode counting algorithms 
%for various
%frequencies, can be viewed in a unified framework. We propose a generic algorithm for automata based serial episode counting, of which, each of the exisiting algorithms 
%are specific
%instances. This unified framework helps us in showing correctness proofs of the various algorithms in a streamlined fashion. It aids in generalizing these various
%counting schemes to episodes with general partial orders. It also helps a better understanding of 
%the anti-monotonicity properties of the various frequency counts.
\end{abstract}

 \section{Introduction}
\label{sec:introduction}
Temporal data mining is concerned with finitely many useful patterns in sequential (symbolic) data streams \cite{morchen07}.
Frequent episode discovery, first introduced in \cite{MTV97}, is a popular framework  
for mining patterns from sequential
data. The framework has been successfully used in many application domains, e.g., analysis of   
alarm sequences in telecommunication networks \cite{MTV97},  
root cause diagnostics from faults log data in manufacturing \cite{USSL09},  
user-behavior prediction from web interaction logs \cite{LTW08},  
inferring functional connectivity from multi-neuronal spike train data \cite{PSU08}, 
relating financial events and stock trends \cite{NF03},
protein sequence classification \cite{BCSZ06},  
intrusion detection \cite{LB00,WWT08}, text mining \cite{ITN04}, seismic data analysis \cite{MR04}  etc.
The data in this framework is a single long stream of events, where each event  
is described by a symbolic event-type from a finite alphabet and the time of occurrence of the event.  
The patterns of interest are termed episodes.
Informally, an episode   
%\footnote{In this paper we only consider the {\em serial} episodes of
%\cite{MTV97}.} 
is a short ordered sequence of event types, and a {\em frequent} episode is one that
occurs often enough in the given data sequence. Discovering frequent episodes is
a good way to unearth temporal correlations in the data.  
Given a user-defined frequency threshold, the task is to efficiently obtain all
frequent episodes in the data sequence.

An important design choice in frequent episode discovery is the definition
of frequency of episodes. Intuitively any frequency should capture the notion  
of the episode occurring many times in the data and,
at the same time, should have an efficient algorithm for computing the same.  
There are many ways to define frequency
and this has given rise to different algorithms for frequent episode discovery
\cite{HC08,ITN04,MT96,MTV97,garriga03,MR04,laxman06}.
In the original framework of \cite{MTV97}, frequency was
defined as the number of fixed-width sliding windows over the data that
contain at least one occurrence of the episode. Another notion for frequency 
 is based on the number of {\em minimal} occurrences \cite{MT96,MTV97}.
Two frequency definitions called {\em head frequency} and {\em total frequency}  
are proposed in \cite{ITN04} in order to overcome some  
limitations of the windows-based frequency of \cite{MTV97}. 
In \cite{laxman06}, two more frequency definitions for episodes were proposed,
based on certain specialized sets of occurrences of episodes in the data.

Many of the algorithms, such as the WINEPI of \cite{MTV97} and the occurrences-based
frequency counting algorithms of \cite{LSU07,LSU05}, employ finite state automata as the  
basic building blocks for recognizing occurrences of episodes in the data sequence. 
An automata-based counting scheme for
minimal occurrences has also been proposed in \cite{DFGGK97}. 

%The algorithms employing a depth-first search of the pattern space (not apriori based) for %discovering frequent episodes under the head frequency count 
% have also been proposed in  \cite{HC08}. 
%%The algorithms for discovering frequent episodes under the head frequency count 
% %employ a depth-first search of the pattern space and are not automata-based \cite{ITN04}. 

%A very efficient algorithm for frequent episode discovery under the non-overlapped
%frequency definition, in which only one automaton was required per candidate episode
%during frequency counting is presented in \cite{LSU07}. 
%%In addition to scaling well to very large data sets, the
%%non-overlapped frequency also has elegant theoretical properties \cite{LSU05}.
%However, the fast algorithm of \cite{LSU07} cannot handle an {\em expiry time}
%constraint for episode occurrences (which is basically an upper bound on the time 
%difference between the first and last events of a valid episode occurrence). 
%An efficient automata-based algorithm for counting non-overlapped occurrences with  
%expiry time constraints has been proposed in \cite{laxman06}. The algorithms under 
%the windows-based frequency, minimal occurrences-based frequency, head and total   
%frequency on the other hand can readily handle expiry time constraints
%prescribed by the user.

The multiplicity of frequency definitions and the associated algorithms for frequent episode 
discovery makes it difficult to compare the different methods. 
In this paper, we present a unified  view of algorithms   
for frequent episode discovery under all the various frequency definitions 
. We present a generic automata-based algorithm for   
 obtaining frequencies of a set of  episodes and 
show that all the currently available algorithms can be obtained as special cases of this method.  
This viewpoint helps in obtaining useful insights regarding the kinds of occurrences  
tracked by the different algorithms. The framework also aids in deriving proofs  
of correctness for the various counting algorithms, many of  which are not currently   
available in literature. 
Our framework also helps in understanding the anti-monotonicity conditions satisfied by different frequencies which is needed for the candidate generation step. 
Our general view can also help in generalizing current algorithms, which can discover 
only serial or parallel episodes, to the case of episodes with general partial orders 
and we briefly comment on this in our conclusions. 

%In the original framework of \cite{MTV97}, an episode is 
% defined as a set of event types with a partial order on it. However, all the   
%currently available algorithms for Frequent Episode Discovery are only for either  
% serial episodes (where the partial order is a total order) or     
%parallel episodes (where the partial order is null). Our unified framework also  
%helps us understand how frequent episode counting algorithms may be generalized   
%to the case of general partial orders. 

The paper is organized as follows. Sec.~\ref{sec:overview}  gives an overview of the  
episode framework and explains all the currently used
frequencies in literature. Sec.~\ref{sec:algorithms} presents our generic algorithm and shows that  
 all current  counting techniques for these various frequencies can be 
derived as special cases.  
 Sec.~\ref{sec:proof-of-correctness} gives proofs of correctness for the various  
counting algorithms utilizing this unified framework.
 Sec.~\ref{sec:candgen} discusses the candidate generation step for all these frequencies.
In Sec.~\ref{sec:discussion} we provide some discussion and concluding remarks.

\section{An overview of frequent episode discovery}
\label{sec:overview}

In this section we briefly review the framework of frequent episode
discovery \cite{MTV97}. The data, referred to as an {\em event sequence}, is denoted by $\mathbb{D} = \langle
(E_1,t_1),$ $(E_2, t_2),$ $\ldots (E_n,t_n) \rangle$, where each pair $(E_i,t_i)$ represents an {\em
event}, and the number of events in the event sequence is $n$. Each $E_i$ is a symbol (or {\em
event-type}) from a finite alphabet, $\scrE$, and $t_i$ is a positive integer representing the time of
occurrence of the $i^\mathrm{th}$ event. The sequence is ordered so that, $t_i\leq 
t_{i+1}$ for all $i = 1, 2, \ldots$. The following is an example event sequence with 10 events:
\begin{eqnarray}
(A,1),(A,2),(B,3),(A,6),(A,7),(C,8),(B,9),(D,11),\nonumber\\
(C,12), (A,13),(B,14),(C,15) 
\label{eq:example-sequence}
\end{eqnarray}
An $N$-node episode, $\alpha$, is defined as a triple, $(V_\alpha, \leq_\alpha,g_\alpha)$,  
where $V_\alpha = \{v_1,v_2,\ldots v_N\}$,
is a collection of $N$ nodes,   
$\leq_\alpha$ is a partial order on $V_\alpha$
and $g_\alpha : V_\alpha \rA \scrE$ is a map that associates each node in $\alpha$ with an event
type from $\scrE$. Thus an episode is a (typically small) collection of
event-types along with an associated partial order. When the order $\leq_\alpha$ is total, $\alpha$
is called a serial episode, and when the order is empty, $\alpha$ is called a parallel episode.
In this paper, we restrict our attention to serial episodes\footnote{From now on, we
will simply use 'episode' to refer to a serial episode.}. Without loss of generality, we can now
assume that the total order on the nodes of $\alpha$ is given by $v_1 \leq_{\alpha} v_2 \leq_{\alpha}\ldots  \leq_{\alpha}v_N$.
For example, consider a 3-node episode
$V_\alpha = \{v_1, v_2, v_3\}$, $g_\alpha(v_1)=A$, $g_\alpha(v_2)=B$,
$g_\alpha(v_3)=C$, with 
$v_1\leq_\alpha v_2 \leq_\alpha v_3$. We denote such an  episode by
$(A \rA B \rA C)$. 
An occurrence of episode $\alpha$ in an event sequence $\mathbb{D}$ is a map
$h:V_\alpha \rA \{1,\ldots,n\}$ such that $g_\alpha(v) = E_{h(v)}$ for
all $v \in V_\alpha$, and for all $v,w \in V_\alpha$ with $v<_\alpha
w$ we have $t_{h(v)} < t_{h(w)}$.  In the example event sequence
$(\ref{eq:example-sequence})$, the events $(A,2)$, $(B,3)$ and $(C,8)$
constitute an occurrence of $(A\rA B\rA C)$ while 
$(B,3)$, $(A,7)$  and $(C,8)$ do not. We use  $\alpha[i]$ to refer to the
$i^{th}$ event-type in $\alpha$. This way, an  $N$-node episode
$\alpha$ can be represented using $(\alpha[1] \rA \alpha[2] \rA \ldots \rA \alpha[N])$.
An episode $\beta$ is said
to be a {\em subepisode} of $\alpha$ (denoted $\beta \preceq \alpha$)
if all the event-types in $\beta$
also appear in $\alpha$, and if their order in $\beta$ is same as that
in $\alpha$. For example, $(A\rA C)$ is a 2-node subepisode of the 
episode $(A\rA B\rA C)$ while $(B\rA A)$ is not.

The {\em frequency} of an episode is some measure of how often it occurs in the event sequence. 
 A frequent episode is one whose frequency exceeds a user-defined
threshold. The task in frequent episode discovery is to find all frequent episodes. 

Given an occurrence $h$ of an $N$-node episode $\alpha$, $(t_{h(v_N)} - t_{h(v_1))}$   
is called the {\em span} of the occurrence. In many applications, one may want to consider 
only those occurrences whose span is below some user-chosen limit. (This is because, occurrences 
constituted by events that are widely separated in time may not represent any underlying 
causative influences). We call any such constraint on span as an {\em expiry-time constraint}. 
The constraint is specified by a threshold, $T_X$, such that occurrences of episodes whose span is 
greater than $T_X$ are not considered while counting the frequency.  

One popular approach to frequent episode discovery is to use an Apriori-style level-wise procedure.
At level $k$ of the procedure, a `candidate generation' step combines frequent episodes of size
$(k-1)$ to build candidates (or potential frequent episodes) of size $k$ using some kind of
anti-monotonicity property (e.g.~frequency of an episode cannot exceed frequency of any of its
subepisodes). The second step at level $k$ is called `frequency counting' in which, the algorithm
counts or computes the frequencies of the candidates and determines which of them are frequent.

%It is important to recognize that the candidate
%generation step depends on the choice of definition of episode frequency
%(and the properties of the chosen frequency definition). The candidate generation step exploits a necessary condition(depending on the
%frequency) 
%In the original framework of frequent episode discovery, episode frequency is defined
%as the number of fixed-width sliding windows over the data in which
%the episode occurs at least once. Both the windows-based frequency and the non-overlapped occurrences-based
%frequency satisfy the antimonotonicity property of subepisodes being at least as
%frequent as the corresponding episodes.
%
%As mentioned earlier, in this paper, we are concerned with three occurrences-based frequency
%definitions for episodes -- the non-overlapped occurrences-based frequency, the minimal
%occurrences-based frequency and the non-interleaved occurrences-based frequency. We present these
%frequency definitions below. In all the definitions to follow in this section, we consider a datastream $\mathbb{D}$ and an episode
%$\alpha$. 

\subsection{Frequencies of episodes}
\label{sec:frequencies-of-episodes}

There are many ways to define the frequency of an episode. Intuitively, any definition
must capture some notion of how often the episode occurs in the data. It  
must also admit an efficient algorithm to obtain the frequencies for a set of episodes. 
Further, to be able to apply a level-wise procedure, we need the frequency definition
to satisfy some anti-monotonicity criterion. 
% such as, frequency of episode does not exceed the frequencies of its subepisodes.  
Additionally,  we would also like the frequency
definition to be conducive to statistical significance analysis.

In this section, we discuss various frequency definitions that have been proposed in  
 literature. (Recall that the data is an event sequence, 
  $\mathbb{D}=\langle(E_1,t_1),\ldots (E_n,t_n)\rangle$).
%For the windows-based frequency to be defined below, we need to assume that all the times of
%occurrences of various event-types are integers.

\begin{definition}
\label{def:windowsbased}
\cite{MTV97} A window on an event sequence, $\mathbb{D}$, is a time interval $[t_s,t_e]$, where $t_s$ and $t_e$ are
positive integers such that $t_s\leq t_n$ and $t_e\geq t_1$.   
The  {\em window width} of $[t_s,t_e]$ is given by $(t_e - t_s)$. Given a user-defined 
window width $T_X$, the {\bf windows-based frequency} of $\alpha$ is the number of 
 windows of width $T_X$ which contain at least one occurrence of $\alpha$.
\end{definition}
For example, in the event sequence (\ref{eq:example-sequence}), there are $5$ windows with window width $5$
which contain an occurrence of $(A\rA B\rA C)$.

\begin{definition}
\cite{MTV97} The time-window of an occurrence, $h$, of $\alpha$ is given by $[t_{h(v_1)}, t_{h(v_N)}]$.
A {\em minimal window} of $\alpha$ is a time-window which contains an occurrence of
$\alpha$, such that no proper sub-window of it contains an occurrence of $\alpha$. An occurrence in a minimal
window is called a minimal occurrence. The {\bf minimal occurrences-based frequency} of $\alpha$ 
in $\mathbb{D}$ (denoted $f_{mi}$) is defined as the number of minimal windows of $\alpha$ in $\mathbb{D}$.
\label{def:minimal}
\end{definition}
%Thus, a minimal occurrence of $\alpha$ looks like $[t_{h_\alpha(v_1)},
%t_{h_\alpha(v_N)}]$ (where $h_\alpha$ is an occurrence of $\alpha$)
%with no proper sub-window of it containing any complete occurrence of
%$\alpha$. 
In the example sequence (\ref{eq:example-sequence}) there are
3 minimal windows of $(A\rA B\rA C)$: $[2,8]$, $[7,12]$ and $[13,15]$.

\begin{definition}
\cite{ITN04} Given a window-width $k$, the {\bf head frequency} of $\alpha$ is the number of windows of
width $k$ which contain an occurrence of $\alpha$ starting at the left-end of the window and is denoted as $f_{h}(\alpha,k)$. 
\label{def:head}
\end{definition}

\begin{definition}
\cite{ITN04} Given a window width $k$, the {\bf total frequency} of $\alpha$, denoted as $f_{tot}(\alpha,k)$,  is defined as follows.
\begin{eqnarray}
f_{tot}(\alpha,k) & = & \min_{\beta \preceq \alpha} f_{h}(\beta,k)
\label{eq:total-frequency}
\end{eqnarray}
\label{def:total}
\end{definition}
For a window-width of $6$, the head frequency $f_{h}(\gamma,6)$ of $\gamma=(A\rA B\rA C)$  in (\ref{eq:example-sequence}) is $4$.
The total frequency of $\gamma$, $f_{tot}(\gamma,k)$, in (\ref{eq:example-sequence}) is $3$ because the
head frequency of $(B\rA C)$ in (\ref{eq:example-sequence}) is $3$.

\begin{definition}
\cite{LSU05} Two occurrences $h_1$ and $h_2$  of $\alpha$ are said to
be {\em non-overlapped} if either $t_{h_1(v_N)}<t_{h_2(v_1)}$ or $t_{h_2(v_N)}<t_{h_1(v_1)}$. A set of
occurrences is said to be non-overlapped if every pair of occurrences in the set is non-overlapped. A set $H$, of 
non-overlapped occurrences of $\alpha$ in $\mathbb{D}$ is {\em maximal} if $|H|\geq |H'|$, where
$H'$ is any other set of non-overlapped occurrences of $\alpha$ in $\mathbb{D}$. The {\bf non-overlapped 
frequency} of $\alpha$ in $\mathbb{D}$  (denoted as $f_{no}$) is defined as the cardinality of 
a maximal non-overlapped set of occurrences of $\alpha$ in $\mathbb{D}$.
\label{def:nonoverlapped}
\end{definition}

Two occurrences are non-overlapped if no event of
one occurrence appears in between events of the other.
The notion of a maximal non-overlapped set is needed since there can be
many sets of non-overlapped occurrences of an episode
 with different cardinality\cite{laxman06}. The non-overlapped frequency of $\gamma$ in 
(\ref{eq:example-sequence}) is $2$. A maximal set of non-overlapped occurrences is 
 $\langle(A,2),(B,3),(C,8)\rangle$ and
$\langle(A,13),(B,14),(C,15)\rangle$.  
% However, any set of non-overlapped occurrences which 
% includes $\langle(A,7), (B,9), (C,15)\rangle$, can have only one occurrence 
%(since no other occurrence of $(A\rA B\rA C)$ can be non-overlapped with  
% it in the event sequence (\ref{eq:example-sequence})).

\begin{definition}
\cite{laxman06} Two occurrences $h_1$ and $h_2$ of $\alpha$ are said to be
{\em non-interleaved} if either $t_{h_2(v_j)} \geq t_{h_1(v_{j+1})},\ 
\ j=1,2,\ldots N-1$ or $t_{h_{1}(v_j)} \geq t_{h_2(v_{j+1})},\ 
\ j=1,2,\ldots N-1$. A set 
of occurrences $H$ of $\alpha$ in $\mathbb{D}$ is {\em non-interleaved} if every pair of occurrences
in the set is non-interleaved. A set  $H$ of non-interleaved occurrences of $\alpha$ in $\mathbb{D}$
is {\bf maximal} if $|H|\geq |H'|$, where $H'$ is any other set of non-interleaved occurrences
of $\alpha$ in $\mathbb{D}$. The {\bf non-interleaved frequency} of $\alpha$ in $\mathbb{D}$
(denoted as $f_{ni}$) is defined as the cardinality of a maximal non-interleaved set of occurrences
of $\alpha$ in $\mathbb{D}$.
\label{def:noninterleaved}
\end{definition}

The occurrences $\langle(A,2),(B,3),(C,8) \rangle$ and $\langle (A,7),(B,9)(C,12) \rangle$ are
non-interleaved (though overlapped) occurrences of $(A\rA B\rA C)$ in $\mathbb{D}$.
Together with $\langle(A,13),(B,14),(C,15) \rangle$, these two occurrences  form a set of
maximal non-interleaved occurrences of $(A\rA B\rA C)$ in (\ref{eq:example-sequence}) and thus $f_{ni}=3$.

\begin{definition}
\cite{laxman06} Two occurrences $h_1$ and $h_2$ of $\alpha$ are said to
be {\em distinct} if they do not share any two events. A set of
occurrences is distinct if every pair of occurrences in it is distinct. A set  $H$ of 
distinct occurrences of $\alpha$ in $\mathbb{D}$ is {\bf maximal} if $|H|\geq |H'|$, where $H'$ is
any other set of distinct occurrences of $\alpha$ in $\mathbb{D}$. The {\bf distinct
occurrences-based frequency} of $\alpha$ in $\mathbb{D}$ (denoted as $f_d$) is the
cardinality of a maximal distinct set of occurrences of $\alpha$ in $\mathbb{D}$.
\label{def:distinct}
\end{definition}
The three occurrences that constituted the maximal non-interleaved occurrences of $(A\rA B\rA C)$
in (\ref{eq:example-sequence}) also form a set of maximal distinct occurrences  
in (\ref{eq:example-sequence}).\\

The first frequency proposed in the literature was the windows based count \cite{MTV97}  
and was originally applied for analyzing alarms in a telecommunication network.
It uses an automata based algorithm called WINEPI for counting. Candidate generation exploits the
anti-monotonicity property that all subepisodes are at least as frequent as the parent episode.  
A statistical significance test for frequent episodes based on the windows-based count
was proposed in \cite{GAS03}. There is also an algorithm for discovering frequent episodes
with a maximum-gap constraint under the windows-based count \cite{garriga03}.

%For an event stream $\mathbb{D}$, let $T_s$ and $T_e$ denote $t_1$ and $t_n$ respectively. 
%The window on an 
%event stream $\mathbb{D}$ is an event sequence $\mathbb{w}=(w,t_s,t_e)$, where $t_s\leq t_n$ and $t_e \geq t_1$, and $w$ consists of
%those pairs $(E,t)$ from $\mathbb{D}$ where $t_s \leq t \leq t_e$. The time span $(t_e - t_s)$ is called the $width$ of the window
%$\mathbb{w}$ and is denoted as $width{\mathbb{w}}$. Let $\scrW(\mathbb{D},win)$ denote the set of all those integer endpoint windows 
%$\mathbb{w}$ on
%$\mathbb{D}$ such that $width(\mathbb{w})=win$. 

The minimal windows based frequency and a 
 level-wise procedure called MINEPI to track minimal windows were also proposed in \cite{MTV97}.
This algorithm 
 has  high space complexity since the exact locations of all the minimal windows of
the various episodes are kept in memory. Nevertheless, it is useful in rule generation.
An efficient automata-based scheme for counting the number of minimal windows (along with a proof of correctness)
was proposed in \cite{DFGGK97}. The problem of statistical significance of minimal windows
was recently addressed in \cite{tatti09}. An algorithm for extracting rules under a maximal gap constraint and based on minimal occurrences has been proposed in
\cite{MR04}.

In the windows-based frequency, 
 the window width is essentially an expiry-time constraint (an upper-bound on the  
 span of the episodes).  However, if the span of 
an occurrence is much smaller than the window width, then its frequency is artificially inflated 
 because the same occurrence will be found in several successive
sliding windows.  The head
frequency measure, proposed in \cite{ITN04}, is a variant of the windows-based count intended to overcome this problem.
Based on the notion of head frequency,  \cite{HC08} presents two algorithms MINEPI+ and EMMA. 
  They also point out how head frequency can be a better
choice for rule generation compared to the windows-based or the minimal windows-based counts. Under
the head frequency count, however, there can be episodes whose frequency is higher than 
 some of their subepisodes (see \cite{ITN04} for details). To circumvent this, \cite{ITN04} propose the idea of
total frequency. Currently, there is no statistical
significance analysis based on head frequency or total frequency.

An efficient automata-based counting algorithm under the non-overlapped frequency measure (along
with a proof of correctness) can be found in \cite{LSU07}. 
 A statistical significance test for the same is proposed in  \cite{LSU05}. 
%They also substantiate how non-overlapped frequency measure is a good test statistic 
% by formally connecting 
% frequent episode discovery and HMM learning.  
 However, the algorithm in \cite{LSU07} 
 does not handle any expiry-time constraints. 
 An efficient automata-based algorithm for 
counting non-overlapped occurrences under expiry-time constraint 
 was proposed in \cite{LSU05,laxman06} though this has higher time and space complexity than 
 the algorithm in \cite{LSU07}. No proofs of
correctness or statistical
significance analysis are available for non-overlapped occurrences under an expiry-time constraint.
Algorithms for frequent episode discovery under the non-interleaved frequency  can be found in \cite{laxman06}. 
No proofs of correctness are available for these algorithms. 

Another frequency measure
we discuss in this paper is based on the idea of distinct occurrences. No algorithms are
available for counting frequencies under this measure. The unified view of automata-based
counting that we will present in this paper can be readily used to design  
algorithms for counting distinct
occurrences of episodes.

\section{Unified View of all the Automata based Algorithms}
\label{sec:algorithms}
%After having discussed most of the popular frequency definitions in the episode framework in the previous section, we now describe the 
%automata-based
%counting algorithms for these various counts and finally tie them up in a unified framework.
%The various counting algorithms differ with regard to factors such as when a new automaton has to be started, when to retire automata, 
%when to
%increment count and when to make state transitions. This view point gives us a generic scheme of handling automata for frequency
%counting, which is one of the main contributions of this paper.

In this section, we present a generic algorithm for obtaining frequencies of episodes under the
different frequency definitions listed in Sec.~\ref{sec:frequencies-of-episodes}. The
basic ingredient in all the algorithms is a simple Finite State Automaton (FSA) that
is used to recognize (or track) an episode's occurrences in the event sequence. 
% For the different frequency counts,cthe algorithms differ with regard to   
%when new automata are initialized, when old ones are removed,  when frequency  
%is incremented and when state transitions are effected. 

 The FSA for recognizing occurrences of $(A\rA B\rA C)$ is illustrated in
Fig.~\ref{fig:automaton}. In general,  an FSA for an $N$-node serial episode
$\alpha=\alpha[1]\rA\alpha[2]\rA\dots\rA\alpha[N]$ has $(N+1)$ states. The first $N$ states are 
represented by a pair $(i,\alpha[i+1])$, $i=0,\dots N-1$. The $(N+1)^{th}$ state is  
$(N, \phi)$ where $\phi$ is a null symbol. Intuitively, if the FSA is in 
 state $(j, \alpha[j+1])$, it means that the FSA has already seen the first $j$ event types of 
this episode and is now waiting for $\alpha[j+1]$; if we now encounter an event of type 
 $\alpha[j+1]$ in the data it can accept it (that is, it can transit to its next state).   
  The start (first) state of the FSA is 
$(0,\alpha[1])$. The $(N+1)^{th}$ state is the accepting  state because   
when an automaton reaches this state, a full occurrence of the episode is tracked.

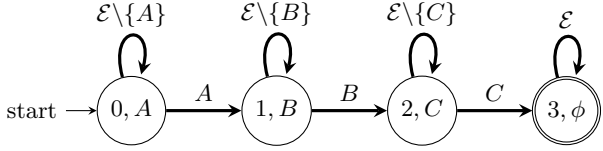
\begin{figure}
\begin{tikzpicture}[>= stealth, node distance=1cm,auto]

\node[state,initial] (0) {$0,A$};
\node[state]         (1) [right=of 0] {$1,B$};
\node[state]         (2) [right=of 1] {$2,C$};
\node[state,accepting]         (3) [right=of 2] {$3,\phi$};
%\node[state]         (5) [right=of 4] {$H$};

\path[->,very thick] 
(0) edge node {$A$} (1)
edge [loop above] node { $\scrE$\textbackslash$\{A\}$ } ()
(1) edge node {$B$} (2)
edge [loop above] node { $\scrE$\textbackslash$\{B\}$ } ()
(2) edge node {$C$} (3)
edge [loop above] node { $\scrE$\textbackslash$\{C\}$ } ()
(3) edge [loop above] node { $\scrE$ } ();
%edge node {$\scrE$} (5)
%(5) edge [loop  above] node {$\scrE$} ();

\end{tikzpicture}
\caption{Automaton for tracking occurrences of $\alpha=(A\rA B\rA C)$}
\label{fig:automaton}
\end{figure}

We first explain how these FSA can be used for obtaining all the different types 
of frequencies of episodes before presenting the generic algorithm. 
While discussing various algorithms, we represent any occurrence $h$ 
 by $[t_{h(v_1)},t_{h(v_2)}\ldots t_{h(v_N)}]$, which is 
 the vector of times of the events that constitute the occurrence.  
For the discussion of all algorithms in this section, we consider the example of tracking 
 occurrences of $\alpha = (A\rA B\rA C\rA D)$ in the data stream $\mathbb{D}_1$ given by  

\begin{eqnarray*}
\mathbb{D}_1 & = & (A,1)(B,3)(A,4)(A,5)(C,7)(B,9)(C,11)(A,14)\\
& & (D,15)(C,16)(B,17)(D,18)(A,19)(C,20)(B,21)\\ 
& & (A,22)(D,23)(B,24)(C,25)(D,29)(C,30)(D,31)
\end{eqnarray*}

There  is a `natural' lexicographic order on the set of
{\bf all} occurrences $\scrH$, of any episode, $\alpha$, defined below. 
This is a total order on $\scrH$ and it will be useful in our analysis. 

\begin{definition}
The lexicographic ordering, $<_{\star}$, on the set
of all occurrences of $\alpha$ is defined as follows: $h_1<_\star
h_2$ if the least $i$ for which $t_{h_1(v_i)}\neq t_{h_2(v_i)}$ is
such that $t_{h_1(v_i)}< t_{h_2(v_i)}$.
\end{definition}

The simplest of all automata-based frequency counting algorithms is the one  
 for counting non-overlapped occurrences \cite{LSU07} which 
uses only $1$-automata per episode. (We call it algorithm NO here). 
 At the start, one automaton for each of the candidate episodes is 
initialized in its start state. Each of the automata make a state transition as soon as a relevant 
 event-type appears in the data stream. Whenever an
automaton reaches its final state, frequency of the corresponding episode is incremented, the
automaton is removed from the system and a fresh automaton for the episode is initialized in the
start state. As is easy to see, this method will count non-overlapped occurrences of episodes.
Under the NO algorithm, we 
denote the occurrence tracked by the $i^{th}$ automaton initialized for $\alpha$ as $h_i^{no}$.

 In our example, algorithm  NO tracks the following two occurrences of the
episode $\alpha$: (i) $h_1^{no}=[1\,3\,7\,15]$ and (ii) $h_2^{no}=[19\, 21\, 25\, 29]$,
and the corresponding non-overlapped frequency is $2$.

In this paper we introduce the concept of {\em earliest transiting} occurrence 
of an episode which is useful for analyzing different frequency counting algorithms. 
% We refer to the occurrences that are tracked by an FSA which makes a  
% state transition as soon as it is  possible, 
% as {\em earliest transiting} occurrences.
\begin{definition}
\em An occurrence $h$ of $\alpha$ is called earliest transiting
if  $E_{h(v_i)}$ is the first occurrence of $\alpha[i]$ after
$t_{h(v_{i-1})}$ $\forall i=2,3\ldots N$.
\label{def:earliest transiting}
\end{definition}

It is easy to see that all occurrences tracked by algorithm NO are earliest transiting. 
Let $\scrH^{e}$ denote the set of all earliest transiting occurrences of a given episode. 
 We denote the $i^{th}$ occurrence (as per the lexicographic ordering of occurrences) in 
$\scrH^{e}$  as $h_i^{e}$.
There are 6 earliest transiting occurrences of $\alpha$ in $\mathbb{D}_1$.  
They are $h_1^{e}=[1\, 3\, 7\, 15]$,
$h_2^{e}=[4\, 9\, 11\, 15]$,
$h_3^{e}=[5\, 9\, 11\, 15]$, $h_4^{e}=[14\, 17\, 20\, 23]$, $h_5^{e}=[19\, 21\, 25\, 29]$ and $h_6^e=[22\, 24\, 25\, 29]$. 
 The earliest transiting occurrences tracked by the NO algorithm are  
$h_1^{no}=h_1^{e}$ and $h_2^{no}=h_5^{e}$.

While the algorithm NO is very simple and efficient, it can not handle any expiry-time constraint. 
 Recall that the expiry-time  
constraint specifies an upper-bound, $T_X$, on the span of any occurrence that is counted.  
Suppose we want to count  with  $T_X = 9$. Both the occurrences tracked by 
NO have spans greater than $9$ and hence the resulting frequency count
would be zero. However, $h_4^e$ is an occurrence which satisfies the expiry time
constraint. Algorithm NO can not track $h_4^e$ because it uses only one automaton per episode 
 and the automaton has to make a state transition as soon as the relevant event-type appears in 
 the data. To overcome this limitation, the algorithm can be modified so that a new
automaton is initialized in the start state, whenever an existing automaton moves out of its start
state. All automata make state transitions as soon as they are possible.
 Each such automaton would track an earliest transiting occurrence. 
In this process, two automata may reach the same state. In our example, after seeing $(A,5)$,
the second and third automata to be initialized for $\alpha$, would be waiting in the same
state (ready to accept the next $B$ in the data). Clearly, both automata will make
state transitions on the same events from now on and so we need to keep only one of them. 
We retain the newer or most recently initialized automaton (in this case, the third automaton)  
 since the span of the occurrence tracked by it would be smaller. 
  When an automaton reaches its final state,  if the
span of the occurrence tracked by it is less than $T_X$, then the corresponding frequency is
incremented and all automata of the episode except the one waiting in  the start state are retired. (This
ensures we are tracking only non-overlapped occurrences). When the occurrence tracked by  
the automaton that reaches the final state  
 fails the expiry constraint, we just retire the current automaton; any other
automata  for the episode will continue to accept events. Under this modified algorithm,  
in $\mathbb{D}_1$, the first automaton that reaches its final state tracks $h_3^e$ 
 which violates the expiry time constraint of $T_X=9$.
So, we drop only this automaton. The next automaton that reaches its final state tracks
$h_4^e$. This occurrence has span less than $T_X=9$. Hence we increment the corresponding
frequency count and retire all current automata for this episode.  
Since there are no other occurrences non-overlapped with $h_4^e$, the 
 final frequency would be 1.  We denote this algorithm for counting 
the non-overlapped occurrences under an expiry-time constraint as NO-X.
The occurrences tracked by both NO and NO-X would be earliest transiting. 

Note that several earliest transiting occurrences may end simultaneously. For example, in
$\mathbb{D}_1$, $h_1^{e}$, $h_2^{e}$ and $h_3^{e}$ all end together at $(D,15)$.  
Both $\{h_2^{e},h_5^{e}\}$ and 
$\{h_3^{e},h_6^{e}\}$ form maximal sets of non-overlapped occurrences. Sometimes (e.g. when
determining the distribution of spans of occurrences for an episode) we would like to
track the {\em innermost} one among the occurrences that are ending together. In this example, 
 this means we want to track the set of occurrences $\{h_3^{e},h_6^{e}\}$. 
 This can be done by simply omitting the
expiry-time check in the NO-X algorithm. (That is, whenever an automaton reaches final state, 
 irrespective of the span of the occurrence tracked by it, we 
increment frequency and retire all other automata except for the one in start state). 
 We denote this as the NO-I algorithm and this 
 is the algorithm proposed in \cite{LSU05}.

In NO-I, if we only retire automata that reached their final states (rather than retire all
automata except the one in the start state), we have an algorithm for counting minimal occurrences
(denoted MO). In our example, the automata tracking $h_3^{e}$, $h_4^{e}$
and $h_6^{e}$ are the ones that reach their final states in this algorithm. The
time-windows of these occurrences constitute the set of all minimal windows of $\alpha$ in
$\mathbb{D}_1$.  Expiry time constraints can
be incorporated by incrementing frequency only  when the occurrence tracked has span less than the
expiry-time threshold. The corresponding expiry-time algorithm is referred to as MO-X.

The windows-based counting algorithm (which we refer to as WB)  
 is also based on tracking earliest transiting
occurrences. WB also uses multiple automata per episode to
track minimal occurrences of episodes like in MO.
The only difference lies in the way frequency is incremented. 
%In addition to the usual frequency threshold, the user also defines a window-width  
% parameter ($T_X$). The algorithm counts the number of sliding windows  
% of width $T_X$ which contain an occurrence of the episode.  
The algorithm essentially remembers, for each candidate episode, the last minimal window in
which the candidate was observed. Then, at each time tick, effectively,  
if this last minimal window lies
within the current sliding window of width $T_X$, frequency is incremented by one. This is because, an occurrence of episode $\alpha$ exists in a given window $w$ if
and only $w$ contains  a minimal window of $\alpha$. 

It is easy to see that head frequency with a window-width of $T_X$ is simply the number of
earliest transiting occurrences whose span is less than $T_X$.   Thus we can have a  
 head frequency counting algorithm
(referred to here as HD)  that is similar to MO-X except
that when two automata reach the same state simultaneously we do not remove the older automaton. 
This way, HD will track all earliest transiting occurrences which
satisfy an expiry time-constraint of $T_X$. For $T_X=10$ and for episode $\alpha$,  
HD tracks $h_3^e$, $h_4^e$, $h_5^e$ and $h_6^e$ 
and returns a frequency count of $4$. The total frequency count for an episode $\alpha$
is the minimum of the head frequencies of all its subepisodes (including itself). This can be
computed as the minimum of the head frequency of $\alpha$ and the total frequency of its $(N-1)$-suffix subepisodes which would have been computed in the previous 
 pass over the data. (See \cite{ITN04} for details).  
The head frequency counting algorithm can have high space-complexity as all the time instants at
which automata make their first state transition need to be remembered.

The non-interleaved frequency counting algorithm (which we refer to as NI) differs  
from the minimal occurrence algorithm
in that, an automaton makes a state transition only if 
there is no other automaton of the
same episode in the destination state. Unlike the other frequency counting algorithms discussed so far,  
such an FSA transition policy
will track occurrences which are not necessarily earliest transiting.
In our example, until the event $(A,4)$ in the data sequence, both the minimal and non-interleaved
algorithms make identical state transitions. However, 
on $(A,5)$, NI will not allow the automaton in state $(0,A)$ to make a state transition as there is
already an active automaton for $\alpha$
in state $(1,B)$ which had accepted $(A,4)$ earlier.
Eventually, NI tracks the occurrences $h_1^{ni}=[1\, 3\, 7\, 15]$, $h_2^{ni}=[4\,9\, 16\,18]$,
$h_3^{ni}= [14\, 17\, 20\, 23]$ and $h_4^{ni}=[19\, 21\, 25\, 29]$. 

 While there are no algorithms reported for counting distinct occurrences, we can construct  
 one using the same ideas. Such an algorithm (to be called as DO) differs  
from the one for counting minimal
occurrences, in allowing multiple automata for an episode to reach the same state. However,
on seeing an event $(E_i,t_i)$ which multiple automata can accept, only one of the automata (the
oldest among those in the same state) is allowed to make a state transition; the others continue to
wait for future events with the same event-type as $E_i$ to make their state transitions. 
The set of maximal distinct occurrences of $\alpha$ in $\mathbb{D}_1$ are
$h_1^d=h_1^e$, $h_2^d=[4\, 9\, 11\, 18]$, $h_3^d=[5\, 17\, 20\, 23]$,  
$h_4^d=[14\, 21\, 25\, 29]$ and $h_5^d=[19\, 24\, 30 \,31]$ which are the ones tracked by this algorithm. 

We can also consider counting {\em all} occurrences of an episode even though it may be 
 inefficient.  
The algorithm for counting {\em all} occurrences 
(referred to as the AO) allows all automata to make transitions whenever the appropriate
events appear in the data sequence. However, at each state transition, a copy of the automaton in
the earlier state is added to the set of active automata for the episode. 
%This algorithm would have 
%a large space complexity because the number of active automata per episode can become very large. 

%summary of choices.
From the above discussion, it is clear that by manipulating the FSA (that recognize occurrences) in 
different ways we get counting schemes for different frequencies. The choices to be made in different 
algorithms essentially concern when to initiate a new automaton in the start state, when
to retire an existing automaton, when to effect a possible state transition and  
when (and by how much) to increment the frequency.  
%main algorithm starts here.
We now present a unified scheme incorporating all this in {\em Algorithm~\ref{algo:unified}} for  
obtaining frequencies of a set of serial episodes. This algorithm has five boolean variables, namely, 
TRANSIT, COPY-AUTOMATON, JOIN-AUTOMATON, INCREMENT-FREQ and RETIRE-AUTOMATON. The counting algorithms 
 for all the different frequencies are obtained from this general algorithm by suitably setting the 
values of these boolean variables (either by some constants or by values calculated using the current 
context in the algorithm). Tables~\ref{tab:transit} -- \ref{tab:retire-automata} specify the choices 
needed to obtain the algorithms for different frequencies. (A list of all   
  algorithms is given in table~\ref{tab:various-counts}). 

As can be seen from our general algorithm, when an event type for which an automaton is waiting 
is encountered in the data, the the automaton can accept it only if the 
variable TRANSIT is true. Hence for all algorithms that track earliest transiting occurrences, 
TRANSIT will be set to true as can be seen from table~\ref{tab:transit}. For algorithms 
NI and DO where we allow the state transition only if some condition is satisfied.  
The condition COPY-AUTOMATON (Table~\ref{tab:copy-automaton}) is for deciding whether or not 
 to leave another automaton in the current state when an automaton is transiting to the next state. 
 Except for NO
and AO, we create such a copy only when the currently transiting automaton
is moving out of its start state. In NO we never make such a copy (because this algorithm 
uses only one automaton per episode) while in AO we need to do it for every state transition. 
As we have seen earlier, in some of the algorithms, when two automata for an episode reach  
the same state, the older automaton is removed. This is controlled by 
JOIN-AUTOMATON, as given by Table~\ref{tab:join-automaton}. 
%Note that this is not relevant 
 %for NO and NI because in these algorithms we will never have two automata reaching the 
%same state. 
INCREMENT-FREQUENCY (Table~\ref{tab:increment-freq}) is the condition under
which the frequency of an episode is incremented when an automaton reaches its final state. 
This increment is always done for algorithms that have no expiry time constraint or window width. 
For the others we increment the frequency only if the occurrence tracked satisfies the constraint. 
RETIRE-AUTOMATA
condition (Table~\ref{tab:retire-automata}) is concerned with removal of all automata of an
episode when a complete occurrence has been tracked.
This condition is true only for the non-overlapped occurrences-based counting algorithms.

Apart from the five boolean variables explained above, our general algorithm contains one 
more variable, namely, INC, which decides the amount by which frequency is incremented when 
an automaton reaches the final state. Its values for different frequency counts are listed 
in~Table~\ref{tab:inc}. 
For all algorithms except WB, we set $INC = 1$. 
 We now explain how frequency is incremented in WB. To count the number of
sliding windows that contain at least one occurrence of the episode, whenever
a new minimal occurrence enters a sliding window, we can calculate the number of consecutive windows
in which this new minimal occurrence will be found in. For example, in $\mathbb{D}_1$,
with a window-width of $T_X=16$, consider
the first minimal occurrence of $(A\rA B\rA C\rA D)$, namely, the occurrence constituted by events
$(A,5)$, $(B,9)$, $(C,11)$ and $(D,15)$. The first sliding window in which this occurrence can be found is $[-1,15]$. The
occurrence stays in consecutive sliding windows, until the sliding window $[5,21]$. When this first
minimal occurrence enters the sliding window $[-1,15]$, we observe that there is no other `older' minimal occurrence
in $[-1,15]$, and hence, as per the {\em else} condition  in {\em Table 
\ref{tab:inc}}, the $INC$ is incremented by $(5-(-1)+1)=7$.  
Similarly, when the second minimal occurrence enters the sliding window $[7,23]$, we increment $INC$ by $(14-7+1=8)$.
%Similarly, the second minimal occurrence
%(constituted by the events $(A,14)$, $(B,17)$, $(C,20)$ and $(D,23)$) first enters the sliding window $[7,23]$.
%At this point, the first minimal occurrence is out of the sliding window and hence $INC=(14-7+1=8)$.
The third minimal occurrence (constituted by the events $(A,22)$,
$(B,24)$, $(C,25)$ and $(D,29)$) first enters the sliding window $[13,29]$, with the second minimal window still
occurring within this window. This third minimal occurrence remains in consecutive sliding windows
until $[22,38]$. As per the {\em if} condition of {\em Table 
\ref{tab:inc}}, $INC$ is incremented by $22-14=8$. We note that such an implementation of the windows-based algorithm
removes the need for the $beginsat(t)$ list of \cite{MTV97} which was used to store all automata whose first
state transition occurred at time-tick $t$.

%generic algorithm
\begin{algorithm}
\caption{Unified Algorithm for counting serial episodes}
\label{algo:unified}
\begin{algorithmic}[1]
\REQUIRE Set $\scrC_N$ of $N$-node serial episodes, event stream $\mathbb{D} =
\langle (E_1,t_1)$, $\ldots$, $(E_n,t_n)) \rangle$, 

\ENSURE Frequencies of episodes in $\scrC_N$
%\REQUIRE same as in {\em Algorithm~\ref{algo-no-serial-1-aut}}
%\ENSURE $\scrF$ as in {\em Algorithm~\ref{algo-no-serial-1-aut}}
%and $\scrW=\{TrueOccWin(\alpha)\, such \, that \, \alpha \,\in \,\scrF\}$.

\FORALL{$\alpha\in\scrC_N$}
	\STATE Add  automaton of $\alpha$ waiting in the start state.
	\STATE Initialize frequency of $\alpha$ to ZERO.
\ENDFOR

\FOR{$l=1$ to $n$}
	\FOR{each automaton, $\scrA$, ready to accept event-type $E_i$}
		\STATE $\alpha$=candidate associated with $\scrA$;
		\STATE $j$ = state which $\scrA$ is ready to transit into;
		\IF{TRANSIT}
			\IF{COPYAUTOMATON}
				\STATE Add Copy of $\scrA$ to collection of automata.
			\ENDIF
			\STATE Transit $\scrA$ to state $j$
			\IF{$\exists$ an earlier automaton of $\alpha$ already in state $j$ but not waiting for $E_i$}
				\IF{JOIN-AUTOMATON}
					\STATE Retain $\scrA$ and retire earlier automaton
				\ENDIF
			\ENDIF
			\IF{$\scrA$ reached final state}
				\STATE Retire $\scrA$.
				\IF{INCREMENT-FREQ}
					\STATE Increment frequency of $\alpha$ by INC.
					\IF{RETIRE-AUTOMATON}
						\STATE Retire all automaton of $\alpha$ and create a state '0' automaton.
					\ENDIF
				\ENDIF
			\ENDIF
		\ENDIF
	\ENDFOR
\ENDFOR
\end{algorithmic}
\end{algorithm}
\begin{table}[] \centering 
\caption{Various frequency counts}
\label{tab:various-counts}
\begin{tabular}{|c|c|} \hline
WB & Windows based \\ \hline
MO & Minimal Occurrences based \\ \hline
MO-X & Minimal Occurrence with Expiry time constraints\\ \hline
NO & Non-overlapped \\ \hline
NO-I & Non-overlapped innermost \\ \hline
NO-X & Non-overlapped with Expiry time constraints \\ \hline
NI & Non-interleaved \\ \hline
DO & Distinct occurrences based \\ \hline
AO & All occurrences based \\ \hline
HD & Head frequency \\ \hline
\end{tabular}

\caption{Conditions for TRANSIT=TRUE}
\label{tab:transit}
\begin{tabular}{|c|c|} \hline
WB, MO, MO-X, HD  & Always \\ 
NO, NO-X, NO-I AO  & \\ \hline
NI & If $\nexists$ earlier  automaton for $\alpha$\\ 
&  in next state $j$ \\ \hline
%&  automaton in state $(j+1)$ which \\%(just before parsing E_i) this condition is also not necessary
%& can also accept the event $(E_i,t_i)$\\ \hline
DO & No other earlier automaton for $\alpha$ \\
& waiting in same state can\\
& transit on event $(E_i,t_i)$. \\ \hline
\end{tabular}

\caption{Conditions for COPY-AUTOMATON=TRUE}
\label{tab:copy-automaton}
\begin{tabular}{|c|c|} \hline
WB, MO, MO-X, HD  & Only if $\scrA$ \\ 
NI, NO-X, NO-I, DO  & is in start state \\ \hline
NO & Never \\ \hline
AO & Always \\ \hline
\end{tabular}

\caption{Conditions for JOIN-AUTOMATON=TRUE}
\label{tab:join-automaton}
\begin{tabular}{|c|c|} \hline
WB, MO, MO-X,  & Always \\ 
NO-X, NO-I &  \\ \hline
%NO, NI & Not Applicable \\ \hline
DO, AO, HD, NO, NI & Never\\ \hline
\end{tabular}

\caption{Conditions for INCREMENT-FREQ=TRUE}
\label{tab:increment-freq}
\begin{tabular}{|c|c|} \hline
MO, NO, NI,   & Always \\ 
DO, AO, NO-I &  \\ \hline
WB, NO-X & If time difference between \\ 
MO-X, HD & first and last state transitions \\ 
& is less than $T_X$(window-width for \\
& WB, expiry time for others) \\ \hline
\end{tabular}

\caption{Conditions for RETIRE-AUTOMATA=TRUE}
\label{tab:retire-automata}
\begin{tabular}{|c|c|} \hline
NO, NO-X, NO-I   & Always \\ \hline
WB, MO, MO-X  & Never  \\ 
HD, NI, DO, AO & \\ 
MO-X & \\ \hline
\end{tabular}

\caption{Values taken by INC}
\label{tab:inc}
\begin{tabular}{|l|} \hline
INC = 1 for all counts except WB.\\
For Windows Based count(WB),\\
If(first window which contains current minimal \\
occurrence also contains the previous minimal\\
occurrence), then \\
INC = Time diff. between start of last window containing \\
 the current minimal occurrence and the start of last \\
window which contains previous minimal occurrence.\\
else\\
INC=time difference between the first and last window \\
containing the current occurrence $+ 1$.\\ \hline
\end{tabular}
\end{table}

\begin{remark}
\label{Remark:3-1} Even though we included AO (for counting all occurrences 
of an episode) for sake of completeness, this is not a good frequency measure. 
%Firstly, it is space-wise very inefficient because we have to keep making copies of 
%automata whenever we make a transition. 
This is mainly because it does not seem to satisfy 
any anti-monotonicity condition. For example, consider the data sequence $<AABBCC>$.  There
are $8$ occurrences of $(A \rA B\rA C)$ but only 4 occurrences of each of its $2$-node subepisodes. 
Also, its space complexity can be high.  
\end{remark}

\begin{remark}
\label{Remark:3.2}: 
The quantitative relationships between the different frequency counts for a given episode can
be described as follows:
\begin{equation}
f_{all}\geq f_{h}\geq f_{tot}\geq f_{d}\geq f_{ni}\geq f_{mi}\geq f_{no}
\label{eq:all-f-relationships}
\end{equation}
where $f_{all}$ denotes the frequency of an episode under AO, while $f_h$ and $f_{tot}$ denote the
corresponding head and total frequencies defined with a window-width exceeding the total time-span
of the event sequence. 
For a large sliding window width, the head frequency $f_h$ is
same as the number of earliest transiting occurrences of an episode. In general, the inequality $f_d\geq f_{ni}$ holds only for injective episodes (An
episode $\alpha$ is injective if it does not contain any repeated event-types). All other inequalities are true for any serial episode.
The first inequality is obvious. The second inequality follows directly from {\em equation \ref{eq:total-frequency}} in {\em
definition \ref{def:total}}. Given a set of $f$ maximal distinct occurrences of an episode $\alpha$ in a data stream $\mathbb{D}$, one can 
extract that many earliest transiting occurrences of not only $\alpha$ but also of all its subepisodes in
$\mathbb{D}$. Hence the third inequality is also true. 
Also, it
is easy to verify that a set of non-interleaved occurrences of an injective episode are also distinct, which validates the fourth 
inequality. We will
show the correctness of the remaining two inequalities in the next section. 
\end{remark}

\section{Proofs of correctness}
\label{sec:proof-of-correctness}

In this section, we present proofs of correctness of the different frequency counting algorithms
presented in Sec.~\ref{sec:algorithms} (all of which are specific instances of {\em Algorithm \ref{algo:unified}}). 

%***************
%In Sec.~\ref{subsec:proof-minimal}, we first present the proof of correctness for  
%MO and show how the non-interleaved frequency is lower bounded by the minimal occurrences-based  
%frequency. In Sec.~\ref{subsec:proof-ETbased}, we present the proofs of correctness  
%of NO-X, NO-I and WB, all of which are closely related to MO.   
%In Sec.~\ref{subsec:proof-nonETbased}, we give proofs for NI and DO and also discuss the issues  
%involved in extending these two algorithms to incorporate expiry time constraints.
%***************

In our proofs, we consider the case of event sequences with distinct occurrence-times for events. 
When we are not considering expiry-time constraints, the actual values of times of occurrences of 
different events are not really important; only the time ordering of  
the events is important in deciding 
 on the occurrences of episodes. Hence, in this section we will use $h(v_1)$  
 interchangeably with $t_{h(v_1)}$,  the 
 time of the first event in the occurrence $h$ and so on.   
% when different events have different times of 
% occurrence, 
% an event sequence, say, $\mathbb{D}=\{(E_1,t_i),\ldots,(E_n,t_n)\}$, can also be 
% viewed as a string of $n$ event-types as:
%$<s=E_1E_2\ldots E_n>$. We refer to $s$ as the event data stream. 
%Recall that the window (or time window) of occurrence $h$ is $[t_{h(v_1)}, t_{h(v_N)}]$. 
%For notational simplicity, in this section, we will denote it as $[h(v_1), h(v_N)]$. 
%Also, since the time of occurrences are distinct we work
%with $h(v_i)$ instead of $t_{h(v_i)}$  wherever possible. 
Modifications needed in the case of data having multiple events with the same  
time of occurrence,  are
discussed at the end of the section.

\subsection{Minimal Window Counting algorithm}
\label{subsec:proof-minimal}

First, we analyze the minimal occurrences counting algorithm (MO). Our proof
methodology is different from the one presented in \cite{DFGGK97}, where, the algorithm is 
viewed as computing a table $S[0\ldots n,0\ldots N]$, where, $S[i,j]$
is the largest value $k\leq i$ such that $E_k\ldots E_i$ contains an occurrence of
$\alpha[1]\rA\ldots \alpha[j]$, using dynamic programming.  
 The algorithm, after processing $E_i$, stores the $i^{th}$ row of this matrix.
The dynamic programming recursion helps compute the $i^{th}$ row of this matrix   
from its $(i-1)^{th}$ row. Whenever
$S[i,N]>S[i-1,N]$, the count is incremented since a new minimal occurrence is recognized. Viewed
from an automata perspective, the $i^{th}$ row of the matrix essentially stores the first state transition times
of the currently active automata. Our analysis of the minimal occurrence algorithm also 
leads to an analysis and proof for counting non-overlapped occurrences (NO and NO-X) as well.  
Another advantage of our proof strategy is that it may be generalized to the case of episodes with
general partial orders. (We briefly discuss this
 in section~\ref{sec:discussion}). 

%We state an  intutively obvious property of earliest transiting occurrences in the following lemma which would be useful 
%in our further
%analysis.

\begin{lemma}
\label{lemma:power}
Suppose $h$ is an earliest transiting occurrence of an $N$-node episode $\alpha$. If $h'$ is any general occurrence such that  $h <_\star h'$, then $h(v_i)\leq h'(v_i)$ $\forall i\,=\,1,2,\ldots N$.
\end{lemma}

This lemma follows easily from the definition of the lexicographic ordering, $<_{\star}$, 
and the definition of earliest transiting occurrence. 

\begin{remark}
\label{remark:4-1}
Recall that $h_i^e$ is the $i^{th}$ earliest transiting (ET)
occurrence of an episode. Thus, by definition,  
 $h_i^e(v_1) < h_j^e(v_1)$ and $h_i^e <_{\star} h_j^e$ whenever $i < j$. 
Hence, from the above lemma, we have $h_i^e(v_k) \leq h_j^e(v_k)$ for all $k$ and $i < j$. 
In particular, we have, $h_i^e(v_1) < h_{i+1}^e(v_1)$ and 
 $h_i^e(v_N)\leq h_{i+1}^e(v_N)$, for an $N$-node episode.
\end{remark}

%************************
%\begin{proof}
%The proof is by induction on $i$.
%From the definition of the lexicographic order, it follows
%that $h(v_1)\leq h'(v_1)$. Suppose the lemma is true for some
%$j<N$, we need to show that $h(v_{j+1})\leq h'(v_{j+1})$.
%Suppose not, then we have the following inequalities.
%\begin{displaymath}
%h(v_j)\leq h'(v_j)<h'(v_{j+1})<h(v_{j+1})
%\end{displaymath}
%
%Hence, we have an occurrence of $g(v_{j+1})$ strictly before
%$E_{h(v_{j+1})}$ and after $E_{h(v_j)}$. This leads to a contradiction since
%$h$ is an earliest transiting occurrence.
%\end{proof}
%*************************** 

The main idea of our proof is that to find 
all minimal windows of an episode, it is enough to capture a certain subset of earliest
transiting occurrences. 

%**************
%For example, consider an ET occurrence $h_i^e$ such that $h_i^e(v_N)=h_{i+1}^e(v_N)$. It is easy to see
%that the window of $h_i^e$, namely $w(h_i^e)$, is not a minimal window.  
%It turns out that the converse of %this is also true.
%****************

\begin{lemma}
\label{lemma:minimal_et}
 An earliest transiting (ET) occurrence $h_i^e$, of an $N$-node episode,  
 is not a minimal occurrence  if and only if 
 $h_i^e(v_N)=h_{i+1}^e(v_N)$.
\end{lemma}

\begin{proof}
The `if' part follows easily from Remark~\ref{remark:4-1}. 
For the `only if' part, let us denote by $w=[n_s,n_e]=[h_i^e(v_1),h_i^e(v_N)]$ the 
window of $h_i^e$. Given that $w$ is not a minimal window, we need to show that 
$h_i^e(v_N)=h_{i+1}^e(v_N)$. Since $w$ is not a minimal window, one of its proper sub-windows contains an occurrence, say, $h$, of this episode. That means if $h$ starts at $n_s$ then it must 
 end before $n_e$.  But, since   
 $h_i^e$ is earliest transiting, 
 any occurrence starting at the same event as $h_i^e$ can not end before $h_i^e$. 
  Thus we must have $h(v_1) > h_i^e(v_1)$. 
%Since $h$ is an occurrence, all the event types of 
%the episode appear in the proper order in the time window $[h(v_1), h(v_N)]$. 
%Since  $h_i^e(v_1) < h(v_1)$ and 
This means, by lemma~\ref{lemma:power}, 
 since $h_i^e$ is earliest transiting, we can not have $h_i^e(v_N) > h(v_N)$. 
Since the window of $h$ has to be contained in the window of $h_i^e$, we thus have $h_i^e(v_N) = h(v_N)$. 
By definition, $h_{i+1}^e$ will start at the earliest possible position after $h_i^e$. Since there 
is an occurrence starting with $h(v_1)$ we must have $h_{i+1}^e(v_1) \leq h(v_1)$. Now, since  
$h_{i+1}^e$ is earliest transiting, it can not end after $h$. Thus we must have  
$h_{i+1}^e(v_N) \leq h(v_N)$. Also,  $h_{i+1}^e$ 
can not end earlier than $h_i^e$ because both are earliest transiting. Thus, we must have 
 $h_i^e(v_N) = h_{i+1}^e(v_N)$. This completes proof of lemma. 
\end{proof} 

%**************
%We prove the `only if' part by contradiction.
%Given that the window of $h_i^e$, $w=[n_s,n_e]=[h_i^e(v_1),h_i^e(v_N)]$ is not a minimal 
%window,  suppose 
%$h_i^e(v_N)< h_{i+1}^e(v_N)$.  Since $w$ is not a minimal window, one of its proper subwindows contains an %occurrence of $\alpha$. If
%this proper subwindow starts at $n_s$, then we have an occurrence $h$ in $w$
%starting at $h_i^e(v_1)$ but ending strictly before $h_i^e(v_N)$. 
%Any earliest transiting (ET) occurrence, $h^e$, is the least occurrence (as per lexicographic ordering %$<_\star$ on $\scrH$)
%among all occurrences starting at $h^e(v_1)$. Here, since $h$ starts at $h_i^e(v_1)$, we have $h<_\star %h_i^e(v_1)$.  But we  also have
%$h_i^e(v_N)>h(v_N)$, which  contradicts {\em Lemma
%\ref{lemma:power}}.
%Similarly, consider the case when the subwindow starts strictly after $n_s$. Any occurrence of $\alpha$ %within this subwindow must
%start cannot start before $h_{i+1}^e(v_1)$
%because the earliest $\alpha[1]$ after $n_s$ is at $h_{i+1}^e(v_1)$. Hence, any such occurrence $h$ satisfies
%$h<_\star h_i^e(v_1)$  because $h_{i+1}^e$ is the least occurrence among all those starting at %$h_{i+1}^e(v_1)$. This occurrence $h$ 
%ends no later than $h_i^e(v_N)$ as it has to lie within $w$.  But by the contradiction hypothesis, we have  %$h_i^e(v_N)< h_{i+1}^e(v_N)$.
%Therefore we finally have $h_{i+1}^e(v_N)>h(v_N)$, which again 
%contradicts {\em Lemma \ref{lemma:power}}.
%\end{proof}
%****************************

\begin{remark}
\label{remark:4-2}
This lemma shows that any ET
 occurrence $h_i^e$ such that $h_i^e(v_N)< h_{i+1}^e(v_N)$ is a minimal occurrence. 
 The converse is also true. Consider a minimal window   
$w=[n_s, n_e]$. Since this is a minimal window, there is an occurrence (and hence an ET occurrence) 
 starting at $n_s$. Denote this ET occurrence by $h_i^e$.   
We know $h_i^e(v_N)=n_e$ because $w$ is a minimal  
window. Then the next ET occurrence $h_{i+1}^e$ has to start 
after $n_s$ and has to end beyond $n_e$ because $w$ is minimal. Thus we have  $h_i^e(v_N)<h_{i+1}^e(v_N)$.
\end{remark}

%********************************** 
% We now exactly know the kind of earliest earliest transiting occurrences whose windows are minimal, i.e. the %window of any ET
% occurrence $h_i^e$ such that $h_i^e(v_N)< h_{i+1}^e(v_N)$ is a minimal window. If we can now show its %converse (i.e. any minimal
% window is of this structure), then counting and
% tracking minimal windows is equivalent to counting all  ET occurrences $h_i^e$, such that $h_i^e(v_N)< %h_{i+1}^e(v_N)$. 
%
%To show the converse, consider a minimal window $w=[n_s,n_e]$. Consider the least occurrence of
%$\alpha$ in $w$. We observe that such an occurrence is always an earliest transiting occurrence, and
%so, we let $h_i^e$ denote the least occurrence of $\alpha$ in $w$.
%The immediate next earliest transiting occurrence has to end beyond $n_e=h_i^e(v_N)$, otherwise $w$ cannot be %minimal. Hence $w$ is a
%window of an earliest transiting occurrence $h_i^e$ such that $h_i^e(v_N)<h_{i+1}^e(v_N)$.
%*****************************

Now we are ready to prove correctness of the MO algorithm. Consider {\em Algorithm \ref{algo:unified}} 
 operating in the MO(minimal occurrence) mode for tracking occurrences of  
 an $N$-node episode $\alpha$.  
Since TRANSIT is always true in the MO mode, all automata would be tracking ET occurrences. 
Since COPY-AUTOMATON is true in MO mode whenever an automaton transits out of start state, we will 
always have an automaton in the start state. {\em This, along with the fact that  
 TRANSIT is always true, 
 implies that the $i^{th}$ initialized automaton would be tracking $h_i^e$,  
 the $i^{th}$ ET occurrence.} 
Let us denote by $\scrA_i^{\alpha}$ the $i^{th}$ initialized automaton.  
However, since JOIN-AUTOMATON is also always true, not all automata (initialized for this episode) 
would result in incrementing the frequency; some of them would be removed when one automaton transits  
 into a state already occupied by some other automaton. In view of Lemma~\ref{lemma:minimal_et} and 
Remark~\ref{remark:4-2}, if we show that the automaton $\scrA_i^{\alpha}$ results in increment of 
frequency if and only if $h_i^e$, the occurrence tracked by it, is such that  
$h_i^e(v_N) < h_{i+1}^e(v_N)$, then, the proof of correctness of MO algorithm is complete.

\begin{lemma}
\label{lemma:minimal_algorithm}
In the MO algorithm the $i^{th}$
automaton that was initialized for $\alpha$, referred to as $\scrA^\alpha_i$, contributes to the frequency 
count iff $h_i^{e}(v_N)<h_{i+1}^{e}(v_N)$.
\end{lemma}

%*******************************
%\begin{proof}
%\[
%\begin{array}{cl}
% &  a^\alpha_i \textrm{does not contribute to the frequency} \\
%%{\LlrA} & a^\alpha_i \textrm{ gets knocked off by some automaton created later}\\
%{\LlrA} & a^\alpha_i \textrm{is removed by a more recently intialized automaton}\\
%{\LlrA} & \exists \,\ a^\alpha_p, (p>i) \textrm{ which transits into a
%state }(q-1),\,\\ 
%& (0<q-1<N)\,\textrm{already occupied by } 
% a^\alpha_i \textrm{ in some} \, \\
% & l^{th} \,\textrm{iteration}\\
%{\LlrA} & \exists \, (k>i,\,1<j<N+1)\, \ni \, h_i^{e}(v_j)=h_k^{e}(v_j). \\
%{\LlrA} & \exists \, (1<j\leq N)\, \ni \, h_i^{e}(v_j)=h_{i+1}^{e}(v_j).\\
%% & (\, \,h_i^{e}(v_j)\leq h_{i+1}^{e}(v_j)\qquad\forall\,1<j\leq N)\\
%
%%{\LlrA} & \forall (1<j\leq N), \hspace{1pt},, h_i(v_j)<h_{i+1}(v_j) \\
% & \textrm{Hence, }  a^\alpha_i \textrm{ contributes to the frequency.} \\
%{\LlrA}& \forall 1<j\leq N, \, h_i^{e}(v_j)<h_{i+1}^{e}(v_j). \\
%{\LlrA}& h_i^{e}(v_N)<h_{i+1}^{e}(v_N)
%\end{array}
%\]
%**************************************

\begin{proof}
\[
\begin{array}{cl}
 &  \scrA^\alpha_i \textrm{ does not contribute to the frequency} \\
{\LrA} & \scrA^\alpha_i \textrm{ is removed by a more recently initialized automaton}\\
{\LrA} & \exists \,\ \scrA^\alpha_k, k>i, \  \textrm{ which transits into a
state } \,\\ 
& \,\textrm{already occupied by } 
 \scrA^\alpha_i  \, \\
{\LrA} & \exists k,j \, s.t.\, k>i,\ 1<j \leq N\, \mbox{and} \, h_i^{e}(v_j)=h_k^{e}(v_j). \\
{\LrA} & \exists j \, 1<j\leq N\, s.t. \, h_i^{e}(v_j)=h_{i+1}^{e}(v_j).\\
& \mbox{~~because, by Remark~\ref{remark:4-1}, for } \: k > i,  \\
&               h_i^e(v_j) \leq h_{i+1}^e(v_j) \leq h_k^e(v_j), \forall j \\
 {\LrA} & h_i^e(v_N) = h_{i+1}^e(v_N)
\end{array}
\]
The last step follows because both $h_i^e$ and $h_{i+1}^e$ are ET occurrences and hence 
 $h_i^e(v_j)=h_{i+1}^e(v_j)$ implies $h_i^e(v_{j'})=h_{i+1}^e(v_{j'}), \ \forall j'>j$.  
%We note here that in the algorithm the JOIN-AUTOMATON condition is checked before checking 
%whether the automaton reached final state (and incrementing frequency) and hence even if $j=N$ 
%in the above, the older automaton will not result in incrementing of the frequency. 

Conversely, we have 
\[
\begin{array}{cl}
 &   \scrA^\alpha_i \textrm{ contributes to the frequency} \\
{\LrA}& \forall j, \  1<j\leq N, \, h_i^{e}(v_j)<h_{i+1}^{e}(v_j) \\
{\LrA}& h_i^{e}(v_N)<h_{i+1}^{e}(v_N).
\end{array}
\]
The first step follows because, if $\scrA^\alpha_i$ contributes to the frequency then 
no automaton initialized after it would ever come to the same state occupied by it and since all 
occurrences tracked are earliest transiting, this must mean $h_i^{e}(v_j)<h_{i+1}^{e}(v_j)$,  
$\forall j$.   
This completes proof of the lemma.
\end{proof}

Another interesting observation is that if $h_i^{e}$ is minimal, then it is non-interleaved with $h_{i+1}^{e}$. 
 Suppose $h_i^{e}$ is minimal and $h_i^{e}$ is not non-interleaved with $h_{i+1}^{e}$.    
Since $h_i^{e}$ is minimal, we have $h_i^{e}(v_{j'})<h_{i+1}^{e}(v_{j'}), \ \forall j'$.  
If $h_i^e$ is not non-interleaved with $h_{i+1}^e$, there exists a $j<N$ such that  
$h_{i+1}^{e}(v_j) < h_i^{e}(v_{j+1})$.
Thus we must have $h_i^e(v_j) < h_{i+1}^{e}(v_j) < h_i^{e}(v_{j+1}) < h_{i+1}^e(v_{j+1})$. 
But this can not be because 
$E_{h_i^{e}(v_{j+1})}$ is the
earliest $\alpha[j+1]$ after $h_i^{e}(v_j)$ and if it is also after $h_{i+1}^{e}(v_{j})$ then   the fact that both $h_i^e$ and $h_{i+1}^e$ are ET occurrences should mean 
 $h_i^{e}(v_{j+1})=h_{i+1}^{e}(v_{j+1})$ which contradicts that $h_i^e$ is minimal. Hence $h_i^e$ and
$h_{i+1}^{e}$ are non-interleaved.

 Thus, given the sequence of  minimal windows, 
the earliest transiting occurrences from each of these minimal windows gives a sequence of
(same number of) non-interleaved occurrences.  
This leads to  $f_{mi}\leq f_{ni}$ as stated earlier in
(\ref{eq:all-f-relationships}).

%\subsection{Other Algorithms tracking Earliest Transiting Occurences}
\subsection{Other ET occurrences-based algorithms}
\label{subsec:proof-ETbased}

\subsubsection{Proofs of correctness for NO-X and NO-I}

The NO-X algorithm can be viewed as a slight modification to the MO algorithm.  
As in the MO algorithm, we always have an automaton in the start state and   
 all automata make transitions 
 as soon as possible and when an automaton transits into a state occupied by another, 
 the older one is removed. However, in the NO-X algorithm, the INCREMENT-FREQ
variable is true only when we have an occurrence satisfying $T_X$ constraint. Hence, to start with, we look for the first minimal occurrence which satisfies the expiry time constraint  
and increment frequency.  At this point, (unlike in the MO algorithm)  
we terminate all automata except the one in the start 
state since we are trying to construct a non-overlapped set of occurrences.  
 Then we look for the next earliest minimal occurrence (which will be non-overlapped 
 with the first one) 
satisfying expiry time constraint and so on.  Since minimal occurrences locally have 
the least time span, this strategy of searching for minimal occurrences satisfying expiry
 time constraint in a non-overlapped fashion is quite intuitive.
Let $H_{nX}=\{h_1^{nX},h_2^{nX}\ldots h_{f'}^{nX}\}$ denote the sequence of 
occurrences tracked by the  
  NO-X algorithm (for an $N$-node episode). Then the following property of $H_{nX}$ is obvious. 
\begin{property}
\label{property:algo3''}
$h_1^{nX}$ is the earliest minimal occurrence satisfying expiry time constraints. $h_i^{nX}$ is 
 the first minimal occurrence (satisfying expiry time constraint) that starts after  
 $h_{i-1}^{nX}(v_N)$. There is no 
minimal occurrence
satisfying expiry time constraint which starts after 
 $h_{f'}^{nX}(v_N)$.
\end{property}

\begin{theorem}
\label{theorem:maximality non-overlap constraints}
\em
$H_{nX}$ is a maximal non-overlapped sequence satisfying expiry time constraint $T_X$.
\end{theorem}
\begin{proof}
Consider any other set of non-overlapped occurrences satisfying expiry constraints, 
$H'$ = $\{h'_1,h'_2\ldots h'_l\}$ such that $h'_i <_\star h'_{i+1}$. 
Let $m\,=\,min\{f',l\}.$ Then we first show 
\[ h_i^{nX}(v_N)\leq h'_i(v_N)\qquad \forall i\,=\,1,2,\ldots m.\]
  Suppose $h'_1(v_N) < h_1^{nX}(v_N)$. Consider the earliest transiting
occurrence $h''$ starting from $h'_1(v_1)$. This ends at or before $h'_1(v_N)$ by lemma \ref{lemma:power}. Among all ET occurrences that end at the same event as $h''$, the last 
  one (under the lexicographic ordering)   
is a minimal occurrence by lemma \ref{lemma:minimal_et}. Its window is contained
in that of $h'_1$ which satisfies the expiry time constraint.  
Hence we have found a minimal occurrence satisfying expiry constraint 
ending before $h_1^{nX}$ which contradicts the first statement of property \ref{property:algo3''}. Hence $h_1^{nX}(v_N)\leq h'_1(v_N)$. 
Now applying the same argument to the data stream starting with the first event 
 after $h_1^{nX}(v_N)$, we get $h_2^{nX}(v_N)\leq h'_2(v_N)$ and so on and thus can 
conclude  $h_i^{nX}(v_N)\leq h'_i(v_N)\,\, \forall i$. This shows that no other set of 
non-overlapped occurrences can have more number of occurrences than those in $H_{nX}$. 
 Hence, $H_{nX}$ is maximal.
\end{proof}

If we choose $T_X$ equal to the time span of the data stream, 
%then we are essentially working with all occurrences. Hence under this
%condition, the NO-X algorithm as per the previous theorem would track a maximal non-overlapped %sequence. Also, 
the NO-X algorithm reduces to the NO-I algorithm because every 
occurrence satisfies expiry constraint. Hence proof of correctness of NO-I algorithm is immediate. 

\subsubsection{Relation between NO-I and NO algorithms}
We now explain the relation between the sets of occurrences tracked  
by the NO and NO-I algorithms. 
 As proved in \cite{LSU07} the NO algorithm (which uses one automaton per episode),  
tracks a maximal non-overlapped sequence of
occurrences, say, $H_{no}=\{h_1^{no},h_2^{no}\ldots h_{f_{no}}^{no}\}$. Since the 
NO-I algorithm has no expiry time constraint, it also tracks a maximal set of 
 non-overlapped occurrences. 
Among all the ET occurrences that end at  $h_i^{no}(v_N)$, let  
 $h_i^{in}$ be the last one (as per the lexicographic ordering). Then the 
$i^{th}$ occurrence tracked by the NO-I algorithm would be $h_i^{in}$ as we show now. 
Since $h_1^{no}$ would be the first ET occurrence, it is clear from our discussion 
in the previous subsection that the first occurrence tracked by the MO algorithm would be 
$h_1^{in}$. As is easy to see, the MO and NO-I algorithms would be identical till the 
first time an automaton reaches the accepting state. Hence $h_1^{in}$ would be the 
first occurrence tracked by the NO-I algorithm. Now the NO-I algorithm would remove 
all automata except for the one in the start state. Hence, it is as if we start the algorithm 
with data starting with the first event after $h_1^{no}(v_N) = h_1^{in}(v_N)$. Now, 
by the property of NO algorithm, $h_2^{no}$ would be the first ET occurrence in this 
data stream and hence $h_2^{in}$ would be the first minimal window here. Hence it is the 
second occurrence tracked by NO-I and so on. 

The above also shows that each occurrence tracked by the NO-I algorithm is also tracked by 
the MO algorithm and hence we have $f_{no} \leq f_{mi}$ as stated 
 in~(\ref{eq:all-f-relationships}). 
$H_{in}$ is also a maximal set of non-overlapping minimal windows as discussed  
in~\cite{tatti09}.

\subsection{Non-interleaved and Distinct Occurrences based Algorithms}
\label{subsec:proof-nonETbased}
The algorithm NI which counts non-interleaved occurrences is different from all the ones  
discussed so far because it does not track ET occurrences. Here also we always have an 
 automaton waiting in the start state. However,  the
transitions are conditional in the sense that the $i^{th}$ created automaton makes a transition from state $(j-1)$ to $j$
provided the $(i-1)^{th}$ created automaton is past state $j$ after processing the current event.  
% That is, on seeing  event of type $\alpha[j]$ in the data stream, 
%the $i^{th}$ created automaton
%is not made to transit if it finds that the $(i-1)^{th}$ created automaton is in $j$ and cannot move %out of it on seeing this event.  
 This is  because we want
the $i^{th}$ automata to track an occurrence non-interleaved with the occurrence tracked by $(i-1)^{th}$ automaton.  
 Let $\scrH_{ni} = \{ h_1^{ni}, h_2^{ni},\dots
h_{f'}^{ni} \}$ be the sequence of occurrences tracked by NI. From the above discussion 
 it is clear that it has the following property (while counting occurrences of $\alpha$). 

\begin{property}
\label{property:algo1'''}
\em
$h_1^{ni}$ is the first or earliest occurrence (of $\alpha$).  
For all $i>1$ and $\forall j=1,\ldots,N-1$, $h_i^{ni}(v_j)$ is the first   
 occurrence of $\alpha[j]$ at or
after $h_{i-1}^{ni}(v_{j+1})$; and  $h_i^{ni}(v_N)$ is the earliest occurrence of $\alpha[N]$ after $h_i^{ni}(v_{N-1})$. There is no occurrence of $\alpha$ beyond 
$h_{f'}^{ni}$ which is non-interleaved with it.
\end{property}

The proof that $H_{ni}$ is a maximal non-interleaved sequence is very similar   
 in spirit to that of the NO-X algorithm. As earlier, we can show that given  
 an arbitrary sequence of non-interleaved occurrences 
$H'$ = $\{h'_1,h'_2\ldots h'_l\}$, we have $ h_i^{ni}(v_k)\leq h'_i(v_k), \  \forall i, k$  
 and hence get the correctness proof of NI algorithm. 
It is easy to verify the correctness of the DO algorithm also along 
 similar lines. 

It appears difficult to extend both the NI and DO algorithms to incorporate  
expiry time constraints. 
 For this we should track a set of occurrences
$h_1,h_2\ldots$ of $\alpha$, where $h_1$ is the first occurrence satisfying $T_X$ 
 and $h_2$ is the next earliest occurrence satisfying $T_X$ that is
non-interleaved with (or distinct from, in case of DO) $h_1$ and so on. 
Note that this $h_2$ need not have to be the earliest occurrence  non-overlapped
with $h_1$. 
%Tracking all ET occurrences and checking which of them satisfy the 
%non-interleaved condition, seems difficult. Also searching in the space of ET occurences is not a right strategy in general for both these counts. 
At present, there are 
no algorithms 
for counting non-interleaved or distinct occurrences satisfying an expiry time 
 constraint.

%In order to efficiently handle the automata corresponding to various candidate episodes at a given level, we use $waits()$ lists as proposed in \ref{MTV97}. The
%$waits()$ lists must be handled in FIFO fashion to  make sure that when two automaton in consecutive states and waiting for the same event type exist, the older one
%is accessed first always.
%--------------if needed -------------
Before ending this section, we 
 briefly outline what needs to be done when the data stream contains multiple events having the same time of occurrence. An
important thing to note is that two events having the same time of occurrence cannot be a part of a serial episode occurrence. Hence,
each automata can at most accept one event from a set of events having the same occurrence time. With this condition, the DO, AO and HD
algorithms go through as before. One would need to process the set of
events having the same occurrence time together and  allow all the permissible automata  to make a one step transition 
first as done using $transitions()$ list in \cite{MTV97}. After this, before processing the set of events with the next occurrence time, we would need to do the multiple automata check
for the various candidate episodes and delete the appropriate older automata for 
algorithms MO, MO-X, NO-I and NO-X.  For the non-interleaved algorithm, one needs to actually back track the transitions which
resulted in two automata to coalesce.

\section{Candidate Generation}
\label{sec:candgen}
In this section, we discuss the anti-monotonicity properties of the various frequency counts, which in-turn are exploited by their respective candidate generation
steps in the Apriori-style level-wise procedure for frequent episode discovery.

It is well known that the windows-based\cite{MTV97}, non-overlapped\cite{LSU05} and total\cite{ITN04} frequency measures satisfy the 
anti-monotonicity property that
{\em all subepisodes of a frequent episode are frequent}. One can verify that the same holds for the distinct occurrences based
frequency too. It has been pointed out in \cite{ITN04} that the head frequency does not
satisfy this anti-monotonicity property. For an episode
$\alpha$, in general, only the subepisodes involving $\alpha[1]$ are as frequent as $\alpha$ under the head count. In a level-wise 
apriori-based
episode discovery, the candidate generation for the head frequency count would exploit the condition that if an $N$-node episode is 
frequent, then all $(N-1)$-node 
subepisodes that include $\alpha[1]$ have to be frequent. 
 The head frequency definition has some 
 limitations in the sense that the frequency of the $(N-1)$-node  
 suffix subepisode\footnote{Given an $N$-node episode 
 $\alpha[1] \rightarrow \alpha[2] \rightarrow \cdots \rightarrow \alpha[N]$,  
its $K$-node {\em prefix subepisode} is 
 $\alpha[1] \rightarrow \alpha[2] \rightarrow \cdots \rightarrow \alpha[k]$
 and its $(N-K)$-node suffix subepisode is  
 $\alpha[K+1] \rightarrow \alpha[K+2] \rightarrow \cdots \rightarrow \alpha[N]$ 
 for $K=1, 2, \cdots, (N-1)$.}
can be  arbitrarily low. Consider the event
stream with $100$ $A$s followed by a $B$ and $C$. Suppose all occurrences of $A\rA B\rA C$ satisfy the expiry constraint $T_X$. Even
though there are $100$ occurrences of $A\rA B \rA C$, there is only one occurrence 
  of $B\rA C$. This can be a problem when one desires that the frequent episodes 
 capture repetitive causative influences.

%\subsection{Frequencies of Subepisodes for the Minimal and Non-Interleaved Counts}

Like the head frequency, 
the minimal occurrences (windows) and the non-interleaved occurrences  
also do not satisfy the anti-monotonicity property
that all subepisodes are at least as frequent as the corresponding episode.  
However, the $(N-1)$-node prefix and suffix subepisodes are at least  
as frequent as the episode as we show below. For an example, consider a data stream where 
successive events are given by $ABACBDCD$.  
 Even though there are two minimal windows (and two non-interleaved occurrences)   
of $A\rA B \rA C\rA D$, there is only one 
minimal window (and one non-interleaved occurrence) of  each of the non-prefix  
and non-suffix subepisodes 
$A\rA B\rA D$ and $A\rA C\rA D$.  However, the situation here is not as bad as 
 that for head frequency because all such subepisodes will have at least as many 
distinct occurrences as the number of minimal or non-interleaved occurrences 
of the episode, at least in case of injective episodes. (Note that this example 
is that of an injective episode). This is because, in case of injective episodes, 
 the number of distinct occurrence is always greater than the non-interleaved count,
which in-turn is greater than the minimal windows count. 
Hence, given that there are $f$ non-interleaved or minimal occurrences of  
an injective episode $\alpha$,
there are at least $f$ distinct occurrences of $\alpha$ too. 
Since the distinct occurrences based
frequency satisfies the original anti-monotonicity
property, all subepisodes of $\alpha$ too will have at least $f$ distinct occurrences.

We now formally prove the
anti-monotonicity property for minimal and non-interleaved occurrences based frequencies. 
\begin{theorem}
If a N-node serial episode $\alpha$ has a frequency $f$ in the minimal or the non-interleaved sense, then its 
$(N-1)$-prefix subepisode $(\alpha_p)$ and suffix subepisode $(\alpha_s)$ 
have a frequency of at least $f$.  
\end{theorem}
%The remaining subepisodes can have a frequency less than $f$. In fact, we can show
%that they are lower bounded by $f/2$.

\begin{proof}
Consider a minimal window of the episode $\alpha$,
 $w = [n_{s}, n_{e}]$. Consider the earliest occurrence $h_p$ of the prefix 
subepisode starting from $n_{s}$ and let $w'$ be its window.   Any
proper sub-window of $w'$ starting at $n_{s}$ and containing an occurrence of $\alpha_p$  contradicts lemma \ref{lemma:power}. A proper sub-window of $w'$ containing an occurrence of $\alpha_p$ starting after 
$n_{s}$ would contradict
the minimality of $w$ itself. Hence $w'$ is a minimal window of $\alpha_p$ starting at $n_{s}$.  
% Since each minimal window of
%$\alpha$ has a unique starting point, each of its prefix minimal windows  
%considered above are distinct.  
We hence conclude that $\alpha_p$ has a frequency of at least $f$. 
A similar proof works for the suffix subepisode by considering the  
 window of the last occurrence $h_s$ of the suffix subepisode ending at $n_{e}$. 

 Let $\scrH_{ni}=\{h_1,h_2,\dots h_f\}$ be a maximal non-interleaved sequence.  
 From each occurrence $h_k$, we choose the 
sub-occurrence $h'_k=[h_k(v_1), h_k(v_2),\dots h_k(v_{N-1}]$, of $\alpha_p$.  It is easy to see that this new
sequence of occurrences ${h'_1, h'_2, \dots h'_f}$ forms a non-interleaved sequence. Hence the frequency of $\alpha_p$ is at least
$f$. A similar argument works for the suffix episode. 
\end{proof}

Hence, for every episode $\alpha$, we extract its $(N-1)$ suffix, go down the candidate list and search for a block of
episodes whose $N-1$ prefix matches this suffix. We form candidates as many as the number of episodes in this matching block. This
kind of candidate generation has already been reported in the literature in \cite{SA96}, \cite{OPS04} and \cite{PSU08}   in the 
context  of
sequences under inter-event time
constraints.

\section{Discussion and Conclusions}
\label{sec:discussion}
The framework of frequent episodes in event streams is a very useful  
data mining technique for unearthing temporal dependencies from data streams in 
many applications. The framework is about a decade old and many different frequency 
measures and associated algorithms have been proposed over the last ten years. In this 
paper we have presented a generic automata-based algorithm for obtaining frequencies 
of a set of candidate episodes. This method unifies all the known algorithms in the sense 
that we can particularize our algorithm (by setting values for a set of variables) for 
counting frequent episodes under any of the frequency measures proposed in literature. 

As we showed here, this unified view gives useful insights into the kind of occurrences  
counted under different frequency definitions and thus also allows us to prove relations 
between frequencies of an episode under different frequency definitions. Our view also 
allows us to get correctness proofs for all algorithms. We introduced the notion of 
earliest transiting occurrences and, using this concept, are able to get simple proofs of 
correctness for most algorithms. This has also allowed us to understand the kind of  
anti-monotonicity properties satisfied by different frequency measures. 

While the main contribution of this paper is this unified view of all frequency counting   
algorithms, some of the specific results presented here are also new.  The relationships between 
 different frequencies of an episode (cf. eqn~\ref{eq:all-f-relationships}), is proved here for the 
 first time.   The distinct-occurrences 
based frequency and an automata-based algorithm for it are novel. 
%To the best of our knowledge, 
%the automata-based counting scheme and the associated proof of correctness for the head frequency 
 %is also not available in literature. 
 The specific proof of correctness presented here for minimal occurrences is also novel.
 Also, the correctness proofs for non-overlapped  
occurrences based frequency counting under expiry time constraint has been provided here for the first time.

In this paper we have considered only the case of serial episodes. This is because, 
 at present, there are no algorithms for discovering general partial orders under the  
 various frequency definitions. However, all 
 counting algorithms explained here for serial episodes can be extended to episodes   
 with a general partial order structure.
 We can come up with a similar finite state automata(FSA) which track the   
 earliest transiting occurrences of an episode with a general partial order
structure \cite{ALRS09}. For example, consider a partial order episode $(A\,B)\rA C$   
 which represents $A$ and $B$ occurring in any order followed by a $C$. 
In order to track an
occurrence of such a pattern, the initial state has to wait for either of $A$ and $B$.  
On seeing an $A$ it goes to  state-1 where it waits only for a $B$; on the other hand, on seeing 
 a $B$ first 
it moves to  state-2 where it  waits only for an $A$. Then on seeing a $B$ in state-1 or  
 seeing a $A$ in state-2 it moves into state-3 where it waits for a $C$ and so on.  
 Thus, in each state in such a FSA, in general, we wait for any of a set of event types 
(instead of a single event for serial episodes) and a given state will now branch out into   
 different  states on different event types.  With such a FSA technique it is possible to 
 generalize the method presented here so that we have algorithms for counting frequencies of  
 general partial order episodes under different frequencies. 
 The proofs presented here for serial episodes can also be extended  
 for general partial order episodes. While it seems possible, as explained above, to generalize  
 the counting schemes to handle general partial order episodes, it is not obvious  what  
 would be an appropriate candidate generation scheme for general partial order episodes 
 under different frequency definitions. This is an important direction for future work.

In this paper, we have considered only expiry time constraint which prescribes  
 an upper bound on the  span of the occurrence. It would be interesting to see under what  
 other time constraints (e.g., gap constraints),   
 design of counting algorithms under this generic framework is possible.    
 Also, some unexplored choice of the boolean conditions in the proposed generic 
 algorithm may give rise to algorithms for new useful frequency measures.  This is also a useful 
 direction of research to explore.

\bibliographystyle{plain}
%\bibliography{srivats-paper}

\end{document}